\documentclass[10pt]{article} 
\usepackage{simpleConference}
\usepackage{times}
\usepackage{graphicx}
\usepackage{amssymb}
\usepackage{url,hyperref}
\usepackage[utf8]{inputenc}
\usepackage{csquotes}
\usepackage{amsmath}
\usepackage{amsfonts}
\usepackage{amsthm}
\usepackage{bbm}
\usepackage{mathrsfs}
\usepackage{adjustbox}
\usepackage{enumitem}
\usepackage{xcolor}
\usepackage{float}
\usepackage{biblatex}
\usepackage{caption}
\usepackage{setspace}
\usepackage{subfigure}
\usepackage{algorithm}
\usepackage{algpseudocode}

\usepackage{svg}
\usepackage{float}
\usepackage{breqn}  
\usepackage{booktabs} 
\usepackage{enumitem} 
\usepackage{palatino} 
\usepackage{xspace}
\def\sc{\textsc}

\theoremstyle{definition}
\newtheorem{definition}{Definition}[section]
\newtheorem{theorem}{Theorem}[section]
\newtheorem{corollary}{Corollary}[section]
\newtheorem{lemma}[theorem]{Lemma}
\begin{document}

\parindent=0mm
\parskip=2mm
\setlist[itemize]{noitemsep, topsep=0pt}

\title{Robust Streaming, Sampling, and a Perspective on Online Learning}
\author{\textbf{Evan Dogariu\quad\quad Jiatong Yu}
\\
Department of Computer Science, Princeton University
\\
\small
\texttt{\{edogariu, jiatongy\}@princeton.edu}
}

\maketitle
\thispagestyle{empty}

\renewenvironment{abstract}
 {\small
  \begin{center}
  \bfseries \abstractname\vspace{-.5em}\vspace{0pt}
  \end{center}
  \list{}{%
    \setlength{\leftmargin}{30mm}
    \setlength{\rightmargin}{\leftmargin}%
  }%
  \item\relax}
 {\endlist}

\setstretch{1.125}
\begin{abstract}
A young and rapidly growing field of theoretical computer science is that of robust streaming. The general subject of streaming faces many use cases in practice, coming up in problems like network traffic analysis and routing, reinforcement learning, database monitoring, server query response, distributed computing, etc. A nascent subfield of streaming concerns streaming algorithms that are robust to adversarially prepared streams, which can be found to have substantial practical grounding. For example, an adversary could submit a small amount of carefully chosen traffic to produce a denial-of-service attack in a network routing system; a robust routing algorithm in this setting would have immense practical use. We investigate this new field of robust streaming and in particular the formalization of robust sampling, which concerns sampling from an adversarially prepared stream to recover a representative sample. 

Throughout this survey, we also highlight, explore, and deepen the connection between the field of robust streaming and that of statistical online learning. On the surface, these fields can appear distinct and are often researched independently; however, there is a deep interrelatedness that can be used to generate new results and intuitons in both places. 

In this work we present an overview of statistical learning, followed by a survey of robust streaming techniques and challenges, culminating in several rigorous results proving the relationship that we motivate and hint at throughout the journey. Furthermore, we unify often disjoint theorems in a shared framework and notation to clarify the deep connections that are discovered. We hope that by approaching these results from a shared perspective, already aware of the technical connections that exist, we can enlighten the study of both fields and perhaps motivate new and previously unconsidered directions of research.
\end{abstract}
\newpage 
\tableofcontents
\newpage

\section{Preliminaries}
\subsection{Statistical Learning}
The field of statistical learning concerns itself with the formal study of questions such as which tasks are learnable from data, how much training data is required to fit a general distribution, and other topics of this flavor. We present formal frameworks for research in this field in the offline and online settings below.
\subsubsection{Offline Setting}
In the offline setting, a learner attempts to learn from a dataset that is available all at once (in contrast with the online setting, where data arrives sequentially in a stream). In order to formalize these ideas, we will supply some notation to the setting of learning from data. We focus on binary classification below and represent our hypothesis as subsets of the domain, but general statistical learning theory follows the same principles and can reduce to this framework.

We consider learning over a set system $(\mathcal{U}, \mathcal{R})$, where $\mathcal{U}$ is a universe of possible elements and $\mathcal{R} \subseteq 2^{\mathcal{U}}$ is a set of subsets of the universe. $\mathcal{R}$ is called the \textit{hypothesis class} in which a learner attempts to learn, for reasons we will see below. Let $\mathcal{S} = ((x_1, y_1), ..., (x_m, y_m)) \in \mathcal{U} \times \{0, 1\}$ be a finite training sample of size $m$ where $x_i$'s are sampled  and labeled either 0 or 1 according to some scheme. To be specific about the data generation, we work in the \textit{realizable} setting if $\mathcal{U}$ is equipped with a distribution over its elements $\mathcal{D}$ and there exists some hypothesis $f \in \mathcal{R}$ that generates the labels (i.e. each $y_i = 1$ if $x_i \in f$ and 0 otherwise). If not, then we are in the \textit{agnostic} setting in which the space $\mathcal{U} \times \{0, 1\}$ is equipped with a joint distribution $\mathcal{D}$. The realizability assumption assumes that there is some perfect ground truth hypothesis in $\mathcal{R}$ that we would like the learner to learn; otherwise, we would like to learn the best hypothesis we can. A \textit{learning algorithm} is any algorithm that, given a training sample $\mathcal{S}$, outputs a hypothesis $h \in \mathcal{R}$ about the training data. Hopefully, the learned hypothesis $h$ not only fits the training sample $\mathcal{S}$ well, but also generalizes well. We define the generalization error (or risk) of a hypothesis $h \in \mathcal{R}$ in the realizable setting to be 
$$L_{\mathcal{D}, f} = \underset{x \sim \mathcal{D}}{Pr}[\mathbbm{1}_{x \in h} \neq \mathbbm{1}_{x \in f}]$$
and in the agnostic setting to be
$$L_{\mathcal{D}} = \underset{(x, y) \sim \mathcal{D}}{Pr}[\mathbbm{1}_{x \in h}\neq y]$$
With all of the above notation out of the way, we can now make headway into the statistical learning theory. The first step in formalization is to define what is meant by a hypothesis class being \textit{learnable}; to this end, we introduce the following definitions.
\begin{definition}[PAC-Learnable]
    A hypothesis class $\mathcal{R}$ is PAC-learnable in the realizable setting if there exists a function $m_{\mathcal{R}}: (0, 1)^2 \rightarrow \mathbb{N}$ and a learning algorithm with the following property: For all accuracy and probability tolerances $\epsilon, \delta \in (0, 1)$, for all distributions $\mathcal{D}$ over $\mathcal{U}$ and all labeling sets $f \in \mathcal{R}$, then when running the learning algorithm on $m \geq m_{\mathcal{R}}(\epsilon, \delta)$ i.i.d samples drawn from $\mathcal{D}$ and labeled by $f$ the algorithm returns a hypothesis $h \in \mathcal{R}$ for which
    $$Pr[L_{\mathcal{D}, f}(h) \leq \epsilon] \geq 1 - \delta$$
\end{definition}
\begin{definition}[Agnostic PAC-Learnable]
    A hypothesis class $\mathcal{R}$ is PAC-learnable in the agnostic setting if there exists a function $m_{\mathcal{R}}: (0, 1)^2 \rightarrow \mathbb{N}$ and a learning algorithm with the following property: For all accuracy and probability tolerances $\epsilon, \delta \in (0, 1)$, for all distributions $\mathcal{D}$ over $\mathcal{U} \times \{0, 1\}$, then when running the learning algorithm on $m \geq m_{\mathcal{R}}(\epsilon, \delta)$ i.i.d samples drawn from $\mathcal{D}$ the algorithm returns a hypothesis $h \in \mathcal{R}$ for which
    $$Pr[L_{\mathcal{D}}(h) \leq \epsilon] \geq 1 - \delta$$
\end{definition}
Here, PAC stands for Probably-Approximately-Correct. Intuitively, these definitions mean that the correct hypothesis can be learned within arbitrary risk tolerance (the Approximately Correct) with high probability (the Probably) using a number of data points that is determined by these tolerances. We can characterize hypothesis classes with two concepts that are defined below (for intuition or examples, consult \cite{statistical_learning_book})
\begin{definition}[Uniform Law of Large Numbers, Sample Complexity]
\label{ulln}
    Let $(\mathcal{U, R})$ be a set system. A hypothesis class $\mathcal{R}$ admits a Uniform Law of Large Numbers if there exists a function $m_{\mathcal{R}}: (0, 1)^2 \rightarrow \mathbb{N}$ such that for any distribution $\mathcal{D}$ over $\mathcal{U}$ and any tolerances $\epsilon, \delta \in (0, 1)$, the empirical error $L_{\mathcal{S}}(h)$ over a sample $\mathcal{S}$ of size $m \geq m_{\mathcal{R}}(\epsilon, \delta)$ satisfies
    $$\forall h \in \mathcal{R} \quad \vert L_{\mathcal{S}}(h) - L_{\mathcal{D}}(h) \vert \leq \epsilon$$
    with probability at least $1 - \delta$. This property is often referred to as \textit{uniform convergence}. We call $m_{\mathcal{R}}(\epsilon, \delta)$ the \textit{sample complexity} of $\mathcal{R}$.
\end{definition}
\begin{definition}
\label{vc_def}[VC-Dimension]
    The VC-dimension of a hypothesis class $\mathcal{R}$, denoted $VCdim(\mathcal{R})$, is the maximal size of a set $C \subset \mathcal{U}$ that can be shattered by $\mathcal{R}$. If $\mathcal{R}$ can shatter sets of arbitrarily large size we say that $\mathcal{R}$ has infinite VC-dimension. Here, shattering has the usual definition where $\mathcal{R}$ shatters $C \subset \mathcal{U}$ if for every labeling of $C$ there exists an element of $\mathcal{R}$ that perfectly describes $C$.
\end{definition}

Equipped with the above tooling, we can now discuss the two fundamental theorems of VC/PAC theory. Proofs of both can be found in Chapter 5 of \cite{statistical_learning_book}.
\begin{theorem}[Fundamental Theorem of PAC Learning]
\label{pac}
    Let $(\mathcal{U},\mathcal{R})$ be a set system. The following are equivalent:
    \begin{enumerate}
        \item $\mathcal{R}$ is PAC learnable in the realizable and agnostic settings.
        \item $\mathcal{R}$ admits a Uniform Law of Large Numbers.
        \item $\mathcal{R}$ has a finite VC-dimension.
    \end{enumerate}
\end{theorem}
\begin{theorem}[Fundamental Theorem of PAC Learning - Quantitiative]
\label{quantitative pac}
    Let $(\mathcal{U},\mathcal{R})$ be a set system with finite VC-dimension $d$. Then, there exist constants $C_1, C_2$ such that:
    \begin{enumerate}
        \item $\mathcal{R}$'s Uniform Law of Large Numbers has sample complexity 
        $$C_1 \frac{d + \log(1/\delta)}{\epsilon^2} \leq m_{\mathcal{R}}(\epsilon, \delta) \leq C_2 \frac{d + \log(1/\delta)}{\epsilon^2}$$
        \item $\mathcal{R}$ is agnostic PAC learnable with sample complexity 
        $$C_1 \frac{d + \log(1/\delta)}{\epsilon^2} \leq m_{\mathcal{R}}(\epsilon, \delta) \leq C_2 \frac{d + \log(1/\delta)}{\epsilon^2}$$
        \item $\mathcal{R}$ is PAC learnable with sample complexity
        $$C_1 \frac{d + \log(1/\delta)}{\epsilon} \leq m_{\mathcal{R}}(\epsilon, \delta) \leq C_2 \frac{d \log(1/\epsilon) + \log(1/\delta)}{\epsilon}$$
    \end{enumerate}
\end{theorem}
These two theorems characterize what types of hypothesis classes $\mathcal{R}$ are learnable, and how much data is required to learn them. We will see that the online setting has well-established results that mirror these.

\subsubsection{Online Setting} 
We now consider the online setting, in which samples of data are given to a learner sequentially in a stream. The learner learns by making predictions on the new data point, after which the true label of the point is revealed and the online learner can learn from it. Owing to the potential for previous training examples to not be useful in the future, we must define online learnability differently in both the realizable and agnostic settings than how we did in the PAC framework. 
\begin{definition}[Mistake Bounds, Realizable Online Learnability]
\label{online def realizable}
    Let $\mathcal{R}$ be a hypothesis class and let $A$ be an online learning algorithm. Given any sequence $\mathcal{S} = ((x_1,f(x_1)),...,(x_T,f(x_T))$, where $T$ is any integer and $f \in \mathcal{R}$ generates the labels, let $M_A(\mathcal{S})$ be the number of mistakes $A$ makes on the sequence $\mathcal{S}$. We denote by $M_A(\mathcal{R})$ the supremum of $M_A(\mathcal{S})$ over all sequences of the above form. A bound of the form $M_A(\mathcal{R}) \leq B < \infty$ is called a \textit{mistake bound}. We say that a hypothesis class $\mathcal{R}$ is online learnable in the realizable setting if there exists an algorithm $A$ for which $M_A(\mathcal{R}) \leq B < \infty$.
\end{definition}
\begin{definition}[Regret, Agnostic Online Learnability]
\label{online def agnostic}
    Let $\mathcal{R}$ be a hypothesis class and let $A$ be an online learning algorithm. Given any sequence $\mathcal{S} = ((x_1,y_1),...,(x_T,y_T))$, where $T$ is any integer, let $A(x_i)$ be the prediction that the algorithm made at step $i$ before being revealed the actual label $y_i$. We define the \textit{regret} to be the difference between the number of mistakes made by the learner and the number of mistakes made by the best hypothesis
    $$R_{\mathcal{S}}(T) = \sum_{i=1}^T \mathbbm{1}_{A(x_i) \neq y_i} - \min_{h \in \mathcal{R}}\sum_{i=1}^T \mathbbm{1}_{h(x_i) \neq y_i}$$
    We say that $\mathcal{R}$ is online learnable in the agnostic setting if the expected regret $\mathbb{E}_{\mathcal{S}}[R_{\mathcal{S}}(T)]$ is $o(T)$, or equivalently if $\lim_{T \rightarrow \infty} \mathbb{E}_{\mathcal{S}}[R_{\mathcal{S}}(T)] / T = 0$ (i.e. the amortized expected regret vanishes). We call $R_T(\mathcal{R})$ the \textit{optimal regret}, where $R_T(\mathcal{R})$ is the regret of the best learner against the worst adversarial sequence.
\end{definition}

There is a combinatorial measure, due to Littlestone \cite{littlestone}, that characterizes online learnability much in the same way that the VC-dimension characterizes offline learnability. 
\begin{definition}[Littlestone Dimension]
    Let $(\mathcal{U, R})$ be a set system. The definition of the Littlestone Dimension of $\mathcal{R}$, denoted $Ldim(\mathcal{R})$, is given using mistake-trees: these are binary decision trees whose internal nodes are labeled by elements of $\mathcal{U}$. Any root-to-leaf path corresponds to a sequence of pairs $(x_1, y_1), . . . , (x_d, y_d)$, where $x_i$ is the label of the $i$’th internal node in the path, and $y_i = 1$ if the $(i + 1)$’th node in the path is the right child of the $i$’th node, and otherwise $y_i = 0$. We say that a tree $T$ is shattered by $\mathcal{R}$ if for any root-to-leaf path $(x_1, y_1), . . . , (x_d, y_d)$ in $T$ there is an $h \in \mathcal{R}$ such that $x_i \in h \iff y_i = 1$ for all $i \leq d$. $Ldim(\mathcal{R})$ is the depth of the largest complete tree shattered by $\mathcal{R}$, with the convention that $Ldim(\emptyset) = -1$.
\end{definition}
The above definition is the online equivalent of the VC-dimension: intuitively, the tree $T$ is the largest set for which any path has a perfect predictor in $\mathcal{R}$, very much like how the set $C$ in Definition \ref{vc_def} is the largest set with a perfect predictor in $\mathcal{R}$ (however, the mistake tree framework encapsulates the sequential nature of the online learning problem). An immediate corollary of this is that $VCdim(\mathcal{R}) \leq Ldim(\mathcal{R})$. The following theorem characterizes online learning, given by Littlestone \cite{littlestone} in the realizable setting and by \cite{BenDavid2009AgnosticOL} in the agnostic setting.
\begin{theorem}
    Let $(\mathcal{U, R})$ be a set system. Then, $\mathcal{R}$ is online learnable in both the realizable and agnostic settings if and only if it has finite Littlestone dimension.
\end{theorem}
Previous work in online learnability theory failed to give an adequate description of a uniform law of large numbers condition on online learnability a la Definition \ref{ulln} nor a tight quantitative bound on sample complexities a la Theorem. \ref{quantitative pac}
\subsection{Streaming}
\label{static stream sec}
In the streaming setting, data of large volume arrives sequentially (and rapidly, which requires efficient computation performance), and it is not realistic to store the entire data stream. Instead, algorithms need one or several scan through the stream and approximately answer queries to the data.  
\subsubsection{Data Stream Models}
An input stream of length $m$ in the format $(a_1, \Delta_1), (a_2,\Delta_2) \dots (a_m, \Delta_m)$ arrives sequentially, which describes an underlying frequency vector $f \in \mathbb{R}^{n}$ where $f_i = \sum_{t:a_t = i} \Delta_t$. Namely, $a_t$ describes the index of update and $\Delta_t$ describes the value of update to the frequency vector. This generic model is also called the \textit{turnstile model} of streaming. A common assumption bounds the maximum coordinate of frequency vector at any time step $t$. Let $f^{(t)}_i = \sum_{j\leq t: a_j = i} \Delta_j$, then $||f^{t}||_{\infty} \leq M$ for some $M > 0$.

Another commonly studied model assumes $\Delta_t \geq 0$, and is called the \textit{insertion-only model}. The above setup could then be simplified as follows. Given an input stream $a_1, a_2 \dots a_m$, the frequency vector $f \in \mathbb{R}^n$ is given by $f_i = |\{j \in [m]: a_j = i\}|$.

The streaming task is to respond to queries regarding the frequency vector $f$ at any timestamp $t$ with close approximation. One measure of success is known as strong tracking.

\begin{definition}[Strong Tracking] 
\label{strong tracking thm}
    Let $f^{(1)}, f^{(2)} \dots f^{(m)}$ be the frequency vectors given the input stream, and let $g: \mathbb{R}^n \rightarrow \mathbb{R}$ be the query function over the frenquency vectors. A streaming algorithm is said to provide $\epsilon$-strong g-tracking if, at each time step $t$, the approximation output $R_t$ satisfies 
    \[|R_t - g(f^{(t)})| \leq \epsilon |g(f^{t})|\]
\end{definition}
\subsubsection{Linear Sketches}
\textit{Linear sketching} is a crucial idea in many streaming algorithm designs. On a very high level, \textit{sketches} achieve dimensionality reduction by generating pseudo-random vectors with (limitedly) independent random variables. Let $\Delta_t \in \mathbb{R}^n$. Given a distribution $\mathcal{M}$ over $r \times n$ matrix space, and an evaluation function $\mathcal{F}: \mathbb{R}^{r\times n} \times \mathbb{R}^r \rightarrow R$ where $R$ is the output space, a linear sketching algorithm draws a (randomized) matrix $A \in \mathcal{M}$. The evaluation function is used to respond to queries by $F(A,A\Delta_t)$.

Consider the second moment estimation problem on turnstile model. Alon et al. \cite{linear_sketch} developed a linear sketch approach as follows. The second moment is defined as $F_2 = \sum_i f[i]^2$. For each $i,j$ define $X_{ij}$ to be a random vector of length $n$, with $\pm 1$-valued, $4$-wise independent random variables as elements. Update $X_{ij} = \langle f, X_{ij} \rangle$, and we have $E[X_{ij}^2] = F_2$ and $Var[X_{ij}^2] \leq 2F_2^2$. Define $Y_i$ to be the average of $X_{i1}, X_{i2}, \dots X_{i, O(log(1/\epsilon^2))}$; by Chebyshev's inequality $P(|Y_i - F_2| > \epsilon F_2) \leq O(1)$. Taking the medium of means, and let $Z$ be the medium of $Y_i$s, we have $P(|Z-F_2| > \epsilon F_2) < \delta$. This algorithm uses $O(\frac{1}{\epsilon^2} log(\frac{1}{\delta}))$ space to achieve $(\epsilon)$-approximation of $F_2$ with probability $1-\delta$.

In fact, the best-known algorithms for any problem in the turnstile model involves linear sketching \cite{sketch_suck}. Such applications include \textsc{frequency moments}, \textsc{heavy hitters}, \textsc{entropy estimation}, etc. S. Muthukrishnan\cite{0204_book} provides a detailed documentation of the aforementioned problems for interested readers.

\subsection{Sampling}
\label{static sampling sec}
Sampling in the streaming setting aims to select a small subset $S$, given sequentially-arrived input stream $X = \{x_1, x_2, \dots x_n\}$, that is representative of the input stream. Deterministic sampling algorithms either fail to use a desirably small space, or involves complicated design tailored toward particular problems \cite{adv_sampling}. Therefore this section will primarily discusses \textit{random sampling}.

The formal notion of being representative of the input stream is defined through $\epsilon$-approximation \cite{geom_book}. Intuitively, $S$ is a good representation of $X$ if any subset $R$ has similar density in $X$ and $S$, as formulated by Definition \ref{e-approx}.
\begin{definition}[$\epsilon$-Approximation]
\label{e-approx}
Let X be a stream and let $(\mathcal{U, R})$. A sample $S$ is an $\epsilon$-approximation of X with respect to $\mathcal{R}$ if, for any subset $R \in \mathcal{R}$, it holds that $\left|\frac{|R \cap X|}{|X|} - \frac{|R \cap S|}{|S|}\right| \leq \epsilon$. 
\end{definition}

An \textit{oblivious} sampler is a sampling algorithm that accepts or rejects an element based on index, independent of the values of the stream. Oblivious samplers are well-studied for a variety of sampling tasks, and are directly applied in the adversarial setting in Section 2. Below, we enumerate some classic oblivious samplers. 

\begin{itemize}
    \item \textbf{Bernoulli Sampler} $Ber(n,p)$ samples elements $i \in [n]$ independently with probability $p$.
    \item \textbf{Uniform Sampler} $Uni(n,k)$ randomly draws $k$ indexes $i_1,i_2 \dots i_k$, $i_j \in [1,n]$. Stream elements with matching indexes are sampled. 
    \item \textbf{Reservoir Sampler} $Res(n,k)$ keeps a sample of size $k$. The first $k$ arrivals of the stream are accepted. For an element with index $i>k$, it is accepted to the sample with probability $\frac{k}{i}$; if accepted, an element currently in the samply will be uniformly drawn to be replaced by the new element.
\end{itemize}

Several results connect sampling complexity in this setting with VC theory \cite{vc_1} \cite{vc_2}.
\begin{theorem}
\label{static-vc relation thm}
    Let $(\mathcal{U}, \mathcal{R})$ be a set system with VC-dimension $d$, and let $X$ be a stream drawn from $\mathcal{U}$. Let $S \subseteq X$ be a subset sampled uniformly at random with size $O \left(\frac{d+log(1/\delta)}{\epsilon^2} \right)$. Then $S$ is an $\epsilon$-approximation of $X$ with probability at least $1-\delta$.
\end{theorem}
Note that this is a restatement in sampling theory language of the sample complexity bound for a Uniform Law of Large Numbers (Definition \ref{ulln}) in the VC-theory (Theorem \ref{quantitative pac}, item 1) .

\section{Adversarial Robustness}
Recall from sections \ref{static stream sec} and \ref{static sampling sec} the key assumption that the input stream $X$ is chosen non-adaptively in advance. This assumption does not hold in many modern applications, where the streaming outputs may implicitly or directly affect the future input stream. An important area of application for streaming algorithms is routing \cite{evolution} \cite{network_entropy}, where adversarial robustness has started to play a key role in evaluation metrics \cite{network_entropy}. This section discusses some recent developments of adversarially robust streaming and sampling algorithms. The framework of analysis also contributes to section \ref{alon paper sec}, where Alon et al. develop the connection between adversarial robustness in sampling and online learnability.

\subsection{Adversarial Setting}
\label{adversarial setting sec}
We can model the adversarial streaming/sampling problem as a two-player game between \textsc{Adversary} and \textsc{algorithm} \cite{adv_sampling} \cite{streaming_framework}. In the most generic setting, \textsc{Adversary} is allowed unbounded computational power and can adaptively choose the next element in the stream based on \textsc{algorithm}'s output history.

For a streaming task, \textsc{algorithm} aims to respond to queries about its frequency vector $f$ in close approximation; for sampling, \textsc{algorithm} aims to admit a sample representative of the stream (i.e. an $\epsilon$-approximation of \textsc{adversary}'s input stream). At a high level, the goal of \textsc{Adversary} is to prevent \textsc{Algorithm} from obtaining good results, and the goal of \textsc{algorithm} is to be robust against this. Each round of the two-player game proceeds as follows.
\begin{enumerate}
    \item \textsc{Adversary} submits an element of the stream to \textsc{algorithm}. The choice of element could (probabilistically) depend on all elements submitted in previous rounds, and observations from \textsc{algorithm}'s output history.
    \item \textsc{Algorithm} update its internal state (estimation of $f$ or accepted sample $S$) based on the new element submitted, and output its response to queries regarding its internal state.
    \item \textsc{Adversary} observes \textsc{Algorithm}'s responses and proceeds to the next round. 
\end{enumerate}
\subsection{Robust Streaming}
Hardt and Woodruff \cite{sketch_suck} show that linear sketching is inherently \textit{non-robust} to adversarial inputs. This necessitates new robust streaming algorithms. This section covers recent developments to \textit{robustify} non-robust algorithms through sketch-switching.

\subsubsection{Vulnerability of Linear Sketching}
Hardt and Woodruff \cite{sketch_suck} in particular show that for problems that are at least as hard as $l_p$-norm estimation, linear sketching algorithms suffer from adversarial inputs\footnote{
Please refer to the original paper for formal proofs of the statements made in the algorithm design summary, which are omitted in this paper for simplicity.
}. 

\begin{definition}[\textsc{GapNorm}]
\label{gapnorm}
    Given input stream $X: \{x_1, x_2 \dots\}$ with each $x_t \in \mathbb{R}^n$, the \textsc{GapNorm}$(B)$ problem requires an algorithm to return $0$ if $||x_t||_2 \leq 1$ and return $1$ if $||x_t||_2 \geq B$ for some parameter $B$. If $x_t$ satisfies neither of the conditions, the output can be arbitrarily $0$ or $1$.
\end{definition}
The \sc{GapNorm} problem requires a slight modification to the data stream model in section \ref{static stream sec}: instead of $\Delta_t \in \mathbb{Z}$, each input from the data stream is denoted as a vector $x_t \in \mathbb{R}^n$. 
\begin{theorem}
    \label{thm sketch suck}
    There exists a randomized adversary that, with high probability, finds a distribution over queries on which linear sketching fails to solve \textsc{GapNorm}$(B)$ with constant probability for $B \geq 2$.
\end{theorem}
\begin{proof}
Consider the following two-player game strategy between \sc{Adversary} and \sc{algorithm}. \sc{algorithm} samples a matrix $A$ from a distribution $\mathcal{M}$. At each round, \sc{adversary} submits $x_t \in \mathbb{R}^n$, and \sc{algorithm} responds with evaluation function $F(A, Ax_t)$. \sc{Adversary} aims to learn the \textit{row space} $R(A)$ of \sc{algorithm}. If \sc{adversary} can successfully learn $R(A)$, the following strategy will force \sc{algorithm} to make mistakes on sufficiently many queries:
\begin{spacing}{1.2}\begin{center} \fbox{\parbox{0.7\textwidth}{
With $1/2$ probability, \sc{adversary} submits the zero vector in $\mathbb{R}^n$.\\ With $1/2$ probability, \sc{adversary} submits a vector $x_t $ in the kernel space of $A$.
}}
\end{center}
\end{spacing}
To learn the row space $R(A)$, \sc{Adversary}'s initial element $x_1$ is drawn from $N(0, \tau I_n)$, a mutivariate Gaussian distribution where $\tau I_n$ is the covariance matrix. From the properties of Gaussian variables, the projection of $x_1$ on \sc{Algorithm}'s sketch matrix $A$ is \textit{spherically symmetric}, and therefore the output will only depend on the the norm of projection $P_{A}(x_1)$. If $x_1$ happens to have a high correlation with $R(A)$, \sc{algorithm} will be tricked to calculate a larger estimation of $||P_A(x_1)||_2$. \cite{sketch_suck} proved the existence of a $\tau$ to make the aforementioned difference sufficiently large.

Therefore, \sc{adversary} can first submit multiple elements $x_t$ drawn from $N(0, \tau I_n)$, and aggregate them into a vector $v_1 \in \mathbb{R}^n$ that is almost entirely contained in $R(A)$. This is accomplished by aggregating $m = poly(n)$ \textit{positively-labeled} vectors $x_t$ from the repeated trials into a matrix $G \in \mathbb{R}^{m \times n}$, and obtain $v_1$ as the top right singular vector of $G$. Then $||P_A(v_1)||_2 \geq 1 - \frac{1}{poly(n)}$, and with a sufficiently large $m$ we can say $v_1$ is almost entirely contained in $R(A)$.

At each iteration of the previously attached scheme, after finding $v_1$ \sc{adversary} draws $x_t$ from  $N(0, \tau I_n)$ within the subspace orthogonal to $\{v_1,v_2 \dots \}$ and runs the attack. This effectively reduces one "unknown" dimension in $R(A)$ every time a new $v$ is found. Repeat the iterations until only a constant number of such unknown dimensions remain, and remaining attacks following the aforementioned strategy of switching between zero and kernel space vectors will trick \sc{algorithm} to make a sufficiently large number of mistakes.

The adversarial algorithm above leads to Theorem \ref{thm sketch suck}. 
\end{proof}

\subsubsection{Sketch Switching}
In light of the adversarial vulnerability of linear sketches, Ben-Eliezer and Yogev \cite{streaming_framework} developed generic robustification schemes for linear sketching. Two techniques were proposed to transform static streaming algorithms into adversarially robust ones. This section will focus on the sketch switching technique for its widespread usage in the static setting and its success in robustifying popular problems like \sc{Distinct Element}, \sc{$F_p$ Estimation}, \sc{Heavy Hitters}, and \sc{Entropy Estimation}.

\textbf{The Algorithm}

\textit{Sketch switching} achieves robustness by keeping multiple copies of strong tracking (Definition \ref{strong tracking thm}) algorithms. The high level goal of sketch switching is to change the current output as few times as possible and, when required to update output, move to different copies so that \sc{Adversary} has no enough information to attack. The detailed algorithm is illustrated in \ref{sketch switching alg}.

As long as the previous output based on $(a_{t-1},\Delta_{t-1})$ is a good multiplicative approximation of the estimate of $(a_t, \Delta_t)$ by the copy in use, \sc{algorithm} outputs the previous output, and internally updates the output state to the current estimation. When the previous output falls out of range with the internal estimation, \sc{algorithm} submits the current estimation given by $(a_t, \Delta_t)$, deactivates the current copy, and activate another copy of static strong tracking algorithms.

\begin{algorithm}[h]
\begin{algorithmic}[1]
\caption{Sketch Switching} \label{sketch switching alg}
\State $\lambda \gets \lambda_{\epsilon/8,m}(g)$
\State Initialize independent copies $A_1, A_2 \dots A_\lambda$ of $(\frac{\epsilon}{8}, \frac{\delta}{\lambda})$-strong $g$-tracking static algorithms
\State $\rho \gets 1$
\State $\hat{g} \gets g(\textbf{0})$
\While {new stream update $(a_t, \Delta_t)$}
\State Insert update $(a_t, \Delta_t)$ to all copies $A_1, A_2 \dots A_\lambda$
\State $y \gets$ current output of $A_\rho$
\If{$\hat{g} \notin (1\pm \epsilon/2)y$ } 
\State $\hat{g} \gets y$
\State $\rho \gets \rho+1$
\EndIf
\State Output estimation $\hat{g}$
\EndWhile
\end{algorithmic}
\end{algorithm}

A central definition in the sketch switching algorithm is \textit{flip number} (Definition \ref{flip def}). Lemma \ref{8 lemma} will come handy in proving the main theorem \ref{sketch switch main theorem}.
\begin{definition}[Flip Number]
\label{flip def}
    Let $Y = \{y_0, y_1, \dots, y_m\}$ be any sequence of real numbers. For $\epsilon\geq 0$, the $\epsilon$-flip number $\lambda_\epsilon(Y)$ is the maximum $k$ for which there exist indices $0 \leq i_1 < i_2 \dots < i_k \leq m$ such that $y_{i_{j-1}} \notin (1 \pm \epsilon)y_{i_j}$ for $j \in [2,k]$.
\end{definition}

\begin{definition}[$(\epsilon,m)$-Flip Number]
   Let $\mathcal{C} \subseteq ([n] \times \mathbb{Z})^m$ be the space of possible input streams. Let $g$ be the query function defined in Definition \ref{strong tracking thm}. Then the $Y = \{y_0, y_1, \dots, y_m\}$ where $y_t = g(f^{(t)})$. The $(\epsilon,m)$-flip number $\lambda_{\epsilon,m}(g)$ of $g$ over $\mathcal{C}$ is the maximum $\epsilon$-flip number of $Y$s generated by all possible input stream sequences of length $m$ in $\mathcal{C}$.
\end{definition}

\begin{lemma}
\label{8 lemma}
    Fix $\epsilon \in (0,1)$. Let $U = \{u_1, u_2 \dots u_m\}$, $V = \{v_1, v_2 \dots v_m\}$, and $W = \{w_0, w_1 \dots w_m\}$ be three sequences of real numbers that satisfy the following 
    \begin{enumerate}
        \item For $i \in [m]$, $v_i = (1 \pm \epsilon/8)u_i$
        \item $w_0 = v_0$
        \item For $i > 0$, if $w_{i-1} = (1\pm \epsilon/2)v_i$, then $w_i = w_{i-1}$. Else $w_i = v_i$.
    \end{enumerate}
    Then $w_i = (1 \pm \epsilon)v_i$ for $i \in [m]$. $\lambda_0(W) \leq \lambda_{\epsilon/8, m}(g)$.
\end{lemma}

\textbf{The Proof}

Now we have enough tools to move on to the main theorem of sketch switching, with a skeleton of the proof from the original paper \cite{streaming_framework}.

\begin{theorem}[Sketch Switching]
\label{sketch switch main theorem}
    Fix a query function $g: \mathbb{R}^n \rightarrow \mathbb{R}$ over frequency vector $f$, and let $A$ be a static streaming algorithm that, for any $\epsilon \in [0,1]$ and $\delta > 0$, uses $L(\epsilon,\delta)$ space and satisfies the $(\epsilon/8,\delta/\lambda)$-strong g-tracking property over frequency vectors at each time step $f^{(1)}, f^{(2)}, \dots, f^{(m)}$. Let $\lambda = \lambda_{\epsilon/8, m (g)}$ be the $(\epsilon,\delta)$-flip number of $g$ over input stream class $\mathcal{C}$. Then Algorithm \ref{sketch switching alg} is an adversarially robust algorithm that returns $(1+\epsilon)$-approximation of $g(f^{(t)})$ at each time step $t$ with probability at least $1-\delta$, and uses $O(\lambda \cdot L(\epsilon/8, \delta/\lambda))$ space.
\end{theorem}
\begin{proof}
Assume a fixed, randomized \sc{algorithm} following the robusification framework in Algorithm \ref{sketch switching alg} with static streaming copies $A_1, A_2 \dots A_\lambda$, each with its own \textit{randomness}. Assume a \textit{deterministic} \sc{adversary} for simplicity. By Yao's minimax principle \cite{Yao1977-tg}, this assumption can easily be relaxed to apply to randomized \sc{adversary}. With this assumption, \sc{Adversary}'s choice of $(a_t, \Delta_t)$ is deterministically defined by $((a_1,\Delta_1) \dots (a_{t-1}, \Delta_{t_1})$ and responses $R_1, R_2 \dots R_{t-1}$.

The skeleton of the proof goes as follows: we first show that up until \sc{algorithm} switches a copy of sketch, its return sequence $Y$ satisfies $(1+\epsilon)$-approximation of $g(f)$ at each time step, and inductively apply the same line of reasoning to each switching to $A_\rho,\; \rho \in [\lambda]$. Union bounding the $\lambda$ copies of $(\epsilon/8, \delta/\lambda)$-strong $g$-tracking copies gives us the $1-\delta$ probability of success. And the remaining proof shows that $\lambda$ copies are indeed sufficient to handle adversarial input stream of length $m$.

\textit{\textbf{Robustness of the first sketch.}} First fix the randomness of $A_1$. Let $u_1^1, u_2^1 \dots u_k^1$ be the $(a,\Delta)$ updates \sc{adversary} would make if \sc{algorithm} were to output $y_0 = g(\textbf{0})$ to every $u$. That is, $u_1^1, u_2^1 \dots u_k^1$ is the sequence such that \sc{algorithm} does not switch copy. Let $k+1$ be the time step that $y_0 \notin (1\pm \epsilon/2)A_1(u_{k+1})$ (i.e. where $y_0$ falls out of range for the first time). Here \sc{algorithm} returns $R_1, R_2 \dots R_k = y_0$ and $R_{k+1} = y_1$. By the definition of $(\epsilon/8,\delta/\lambda)$-strong g-tracking, we know that $A_1(u_t) \in (1 \pm \epsilon/8) g(f^{(t)})$ for $t \leq k$, and by design of Algorithm \ref{sketch switching alg} $R_t = y_0 \in (1\pm \epsilon/2)A_1(u_t)$. By Lemma \ref{8 lemma}, we have 
\[R_1,R_2 \dots R_k = y_0 \in (1\pm \epsilon)g(f^{(t)}),\;\;\forall t \leq k\]

\textit{\textbf{Induction on $A_2$.}} Following the definition above, \sc{algorithm} switches to copy $A_2$ at time $t = k+1$. Here $R_{k+1}$, the output of \sc{algorithm}, is updated to be $y_1 = A_1(u_{k+1})$. Recall that by strong tracking, the switching point output $R_{k+1} = y_1 \in (1 \pm \epsilon/8)g(f^{(k+1)})$ by Definition \ref{strong tracking thm}. Consider $\hat{X} = \{u_1^1, u_2^1 \dots u_k^1, u_1^2, u_2^2 \dots u_{k_2}^2\}$ be the concatenation between the aforementioned inputs and the stream of inputs such that \sc{algorithm} will keep outputing $y_1$. Here $k_2$ is the index where $y_1$ starts to fall out of range. Given that the $\epsilon/8$-strong $g$-tracking guarantee should hold on this fixed size of input and, similar to above reasoning, have 
\[R_{k+1}, R_{k+2} \dots R_{k_2} = y_1 \in (1\pm \epsilon) g(f^{t})\;\; \forall t \in [k+1, k_2]\]
Noted that this line of reasoning extends to any $A_\rho$ for $\rho \in [\lambda]$. Then at any time $t$, \sc{algorithm} outputs $y_\rho \in (1 \pm \epsilon)g(f^{(t)})$ except for probability $\frac{\delta}{\lambda}$. Taking a union bound over all copies of $A_\rho$, this gives us the desired $1-\delta$ probability in Theorem \ref{sketch switch main theorem}.

\textit{\textbf{Bounding number of copies required.}} The last step in the proof shows that $\lambda$ samples of static streaming algorithm with strong tracking guarantee suffices to handle an input stream of length $m$. Define $U = \{g(f^{(0)}), g(f^{(1)}), \dots g(f^{(m)})\}$, $V = \{g(f^{(0)}), A_1(u_1), \dots A_1(u_k), A_2(u_{k_1}), \dots A_2(u_{k_2}), \dots\}$, and $W = \{y_0, y_0 \dots, y_1, y_1 \dots \}$ (i.e. \sc{algorithm}'s outputs $R_1, R_2 \dots R_m$). These three sequences satisfy the condition of Lemma \ref{8 lemma}, and thus we have $\lambda_0(W) \leq \lambda_{\epsilon/8,m}(g)$. 
\end{proof}
\textbf{An Example}

Below is a demonstration of how to apply the sketch switching technique to a common streaming problem, \sc{Distinct Elements}. Also known as $F_0$ estimation, the problem is defined through the query function $g(f): |\{i: f[i] \neq 0\}|,\;\; f\in \mathbb{R}^n$,. Although the proof will be omitted, the following Corollary \ref{insertion-only const} plays a central role in applying sketch switching on a variety of tasks. And the constraint of \textit{insertion-only} model in \ref{insertion-only const} naturally constraints the technique itself to insertion-only models.
\begin{corollary}
    \label{insertion-only const}
    Let $p \geq 0$. Assume the input stream has length $m = poly(n)$. The $(\epsilon,m)$-flip number of of $||f||_p^p$ in the insertion-only model is $\lambda_{\epsilon,m}(||f||_p^p) = O(\frac{\log n}{\epsilon})$ for $p\leq 2$, and $O(\frac{p \log n}{\epsilon})$ for $p > 2$. For $p = 0$, we have $O(\frac{\log m}{\epsilon})$. 
\end{corollary}
\begin{theorem}[Robust \sc{DistinctElements}]
    There exists an adversarially robust \sc{algorithm} that outputs $R_t \in (1 \pm \epsilon)||f^{(t)}||_0$ at each time step $t$ with probability at least $1-\delta$. The space usage is given below.
    \[O\left(\frac{\log 1/\epsilon}{\epsilon}\left(\frac{\log 1/\epsilon + \log 1/\delta + \log(\log n)}{\epsilon^2}+ \log n\right)\right)\]
\end{theorem}
\begin{proof}
    First observe that \sc{DistinctElements} inherently describes an insertion-only model. Therefore \sc{Algorithm} will keep $\lambda_{\epsilon,m} = O(\frac{\log n}{\epsilon})$ following Corollary \ref{insertion-only const}. The copies $A_1, A_2, \dots A_\lambda$ uses the static streaming algorithm proposed by Błasiok\cite{Blasiok2018-so}, which maintains a $1-\delta_0$ probability of $\epsilon$-strong $g$-tracking. Each copy $A_\rho$ takes $O\left(\frac{\log 1/\epsilon + \log 1/\delta + \log(\log n)}{\epsilon^2}+ \log n\right)$ bits of space. To avoid complexity blow-up by the multiplicative term $\frac{\log n}{\epsilon}$, use the following trick: instead of discarding a copy permanently, \sc{algorithm} keeps a smaller amount of copy than prescribed by $\lambda_{\epsilon,m}$, and instead achieve the same result in a cyclic manner. After running out of new copies, re-activate a previously deactivated copy, and so on. This in effect reduces copy complexity from $O(\frac{\log n}{\epsilon})$ to $O(\log(1/\epsilon) / \epsilon))$.
\end{proof}

\subsection{Robust Sampling}
\label{robust sampling sec}
The oblivious random samplers discussed in Section \ref{static sampling sec} are used extensively in modern data-intensive systems, such as database monitoring and distributed machine learning. This naturally raises the question of whether the random samplers are vulnerable to potential attacks by adversarial streams.

Recall Theorem \ref{static-vc relation thm}: For a set system $(\mathcal{U}, \mathcal{R})$ with VC-dimension $d$, a random sample of size $O(\frac{d+ log\,1/\delta}{\epsilon})$ is an $\epsilon$-approximation of $X$ drawn from $\mathcal{U}$ with probability $1-\delta$. This well-established connection between sample size and learnability is extensively used in theoretical machine learning and is a consequence of the quantitative version of the Fundamental Theorem of PAC Learning (Theorem \ref{quantitative pac}).

Ben-Eliezer and Yogev \cite{adv_sampling} showed that the sample size in Theorem \ref{static-vc relation thm} is \textit{not} robust to adversarial streams, and developed the minimum overhead required for an adversarially-robust sample. Section \ref{alon paper sec} shows that this overhead is implicitly connected to online learnability and further develops on this result.

Consider the adversarial setting described in Section \ref{adversarial setting sec}. A formal definition of an adversarially robust algorithm is as follows.
\begin{definition}[$(\epsilon,\delta)$-Robust] 
\label{ed robust}
    A randomized sampling algorithm is $(\epsilon,\delta)$-robust if, for any strategy played by the (computationally unbounded) \sc{Adversary}, the sample $S$ achieves $\epsilon$-approximation of input stream $X$ with respect to the set system $(\mathcal{U},\mathcal{R})$ with probability $1-\delta$.
\end{definition}

\subsubsection{Adversarial Attack on Sampling}
\label{adversarial attack on sampling}
\begin{theorem}
    There exists a set system $(\mathcal{U},\mathcal{R})$ with VC-dimension $1$ where, for $\epsilon>0$, $\delta<1/2$, and some constant $c > 0$, the sample size required for static setting is not $(\epsilon,\delta)$-robust. 
\end{theorem}
\begin{proof}
    Consider the set system $(\mathcal{U},\mathcal{R})$ where $\mathcal{U}$ is a well-ordered set $\{1,2,\dots N\},\;\;N \in [n^{6ln\,n}, 2^{n/2}]$. Here $n$ is the size of adversarial input stream. Define $\mathcal{R} = \{[1,b]:\,b\in \mathcal{U}\}$ to be the set of inclusive intervals from 1 to each element of $\mathcal{U}$. 

    Observe that $(\mathcal{U},\mathcal{R})$ has VC-dimension $1$ (recall Definition \ref{vc_def} for definitions of VC-dimension and shattering). To see this, note that any subset $C \subset \mathcal{U}$ of size 1 can clearly be shattered (if $C = \{c\}$, $[1, c]$ describes $C$ for a positive label and $[1, c-1]$ describes $C$ for a negative label of $c$). So, the VC-dimension is certainly at least 1. However, consider any arbitrary set $C = \{c_1, c_2\} \subset \mathcal{U}$ of size 2, and suppose without loss of generality that $c_1 < c_2$. There exists a labeling of $C$ (namely, $c_1$ has a negative label but $c_2$ has a positive label) that no interval in $\mathcal{R}$ can describe (any interval not containing $c_1$ and satisfying its negative label \textit{cannot} contain $c_2$). Similar logic shows that $\mathcal{R}$ cannot shatter any subset of $\mathcal{U}$ of size at least 2. Therefore, we must have $VCdim(\mathcal{R}) < 2 \implies VCdim(\mathcal{R})=1$.

    Here we consider a Bernoulli sampler $Ber(n,p)$ for the following adversarial scheme. It could be easily extended to the other aforementioned types of oblivious random sampling algorithms (e.g. the Reservoir sampler). Assume $p\leq \frac{1}{4}$, a reasonable assumption if the sampler is to obtain a sample size sublinear of a large input stream $X$.

    \begin{algorithm}[h]
    \begin{algorithmic}[1]
    \caption{Adversarial Strategy for a Bernoulli Sampler} \label{bernoulli adv alg}
    \State $a_1 \gets 1$
    \State $b_1 \gets N$
    \State $p' \gets max\{p,\,\frac{ln\,n}{n}\}$
    \For {$i = 1, 2, \dots n$}
    \State $x_i = \lfloor{a_i + (1-p')(b_i - a_i)}\rfloor$
    \If {\sc{Sampler} accepts $x_i$}
    \State $a_{i+1} \gets x_i$
    \State $b_{i+1} \gets b_i$
    \Else 
    \State $a_{i+1} \gets a_i$
    \State $b_{i+1} \gets x_i$
    \EndIf
    \EndFor
    \end{algorithmic}
    \end{algorithm}

    The goal of the adversarial algorithm above is to maintain the following invariant. At round $i$ of the two-player game, 
    \begin{itemize}
        \item Elements that have been accepted by \sc{Sampler} are at most $a_i$
        \item Elements that have been rejected by \sc{Sampler} are at least $b_i$
        \item Elements submitted by \sc{Adversary} at round $i$ are between $a_i$ and $b_i$.
    \end{itemize}
    This invariant ensures that all elements in $S$ are smaller than the rejected elements in $X'=X \setminus S$. Then it naturally follows that $S$ can not be a $\epsilon$-approximation of $X$. Formally, let $s$ be the maximum element in $S$. Now consider $R = [1,s] \in \mathcal{R}$. Observe that $\frac{|R \cap S|}{|S|} = 1$. 
    \[ \left|\frac{|R \cap S|}{|S|} - \frac{|R \cap X|}{|X|} \right| \geq 1 - \frac{|S|}{n} \geq 1 - 2p' \geq 1/2 \geq \epsilon\]
    Noted that the above attack scheme \textit{guarantees} success as long as \sc{Adversary} does not run out of elements to draw from. That is, as long as $a_i < b_i$ for all rounds $i$. The expected sample size drawn by the Bernoulli sampler is $np \leq np'$. By Markov's Inequality $Pr[|S| \geq 2np'] < \frac{1}{2}$. When $|S| < 2np'$, the following lemma from \cite{adv_sampling} completes the claim that \sc{Adversary} will have enough elements to draw from. 
    \begin{lemma}
        If $|S| < 2np'$, then $b_i - a_i \geq n$ for any $i \in [n]$. 
    \end{lemma}
    \end{proof}
\subsubsection{Adversarially Robust Sample Size}
Ben-Eliezer and Yogev\cite{adv_sampling} developed a upper bound on the $(\epsilon,\delta)$-robustness (Definition \ref{ed robust}) of Bernoulli and Reservoir samplers. Similar to the adversarial attack in the previous section, we will focus on the Bernoulli sampler case, while a very similar formulation can be applied to the Reservoir sampler case. 
\begin{theorem}[Robust Sample Size]
\label{main sampling thm}
    For any set system $(\mathcal{U}, \mathcal{R})$, $\epsilon,\delta \in (0,1)$, and given an input stream of size $n$, a Bernoulli sampler $Ber(n,p)$ is $(\epsilon,\delta)$-robust if the following holds.
    \[p \geq 10 \cdot \frac{ln|\mathcal{R}| + ln(4/\delta)}{\epsilon^2 n}\]
\end{theorem}

\begin{corollary}
\label{corollary}
    The sample size drawn by the Bernoulli sampler is well-concentrated near its expectation $np$. The above theorem then implies that a sample size of $\Theta \left( \frac{ln | \mathcal{R}|+ln(1/\delta)}{\epsilon^2}\right)$ is an $\epsilon$-approximation with probability $1-\delta$.
\end{corollary}

\begin{lemma}
\label{main sampling lemma}
    Given a universe $\mathcal{U}$ and a subset $R\subseteq \mathcal{U}$, and let $X$ be the adversarial input stream of length $n$. A Bernoulli sampler with parameter $p \geq 10\cdot \frac{ln(4/\delta)}{\epsilon^2 n}$ satisfies 
    \[Pr \left[ \left|\frac{|R \cap S|}{|S|} - \frac{|R \cap X|}{|X|} \right| \geq \epsilon \right] \leq \delta \] 
\end{lemma}
Noted that main theorem \ref{main sampling thm} naturally follows from Lemma \ref{main sampling lemma} by taking a union bound. Let $(\mathcal{U}, \mathcal{R})$, $\epsilon$, $\delta$ be the ones defined in Theorem \ref{main sampling thm}. Take a Bernoulli sampler with $p = 10 \cdot \frac{ln|\mathcal{R}| + ln(4/\delta)}{\epsilon^2 n}$ that satisfies the condition in the main theorem. For each $R \in \mathcal{R}$, plug in the lemma with parameters $\epsilon,\delta/|\mathcal{R}|$ and obtain the following result:
\[Pr \left[ \left| \frac{|R \cap S|}{|S|} - \frac{|R \cap X|}{|X|} \right| \geq \epsilon \right] \leq \frac{\delta}{|\mathcal{R}|}\]
If for all $R \in \mathcal{R}$ the above event succeeds, it naturally follows that the sample $S$ is an $\epsilon$-approximation of $X$. A union bound over all $R \in \mathcal{R}$ concludes that the Bernoulli sampler with $p$ defined above is $(\epsilon,\delta)$-robust. It remains to prove Lemma \ref{main sampling lemma} to complete the main theorem. 

\begin{proof}
    A major difference in the analysis between the static setting and the adversarial setting here is the assumption of independence between elements in the input stream. Since \sc{adversary} submits elements based on the history of elements submitted and the sample maintained, concentration inequalities like the Chernoff bound can not be applied in the adversarial setting. This motivates the formulation of random variables as a martingale. Given $R \subseteq \mathcal{U}$, for each round $i$ define $X_i$ and $R_i$ to be the input stream and sample maintained up until round $i$. Define the following random variables:
    \[A_i^R = \frac{i}{n} \cdot \frac{|R \cap X_i|}{|X_i|} = \frac{|R \cap X_i|}{n}\]
    \[B_i^R = \frac{|R \cap S_i|}{np}\]
    \[Z_i^R = B_i^R - A_i^R\]

   Observe that the sequence of $Z_i^R$ is a martingale (Lemma \ref{martingale app lemma}). Note that $A_n^R = \frac{|R \cap X|}{|X|}$ and therefore $\left|A_n^R - \frac{|R \cap S|}{|S|}\right|$ is the subject of interest in Lemma \ref{main sampling lemma}. Therefore it suffices to prove the following two inequalities and applying the \textit{triangle inequality} on them:
    \begin{enumerate}
        \item $Pr \left[ |A_n^R - B_n^R | \geq \frac{\epsilon}{2}\right] \leq \frac{\delta}{2}$
        \item $Pr \left[ |B_n^R - \frac{|R \cap S_n|}{|S_n|}| \geq \frac{\epsilon}{2}\right] \leq \frac{\delta}{2}$
    \end{enumerate}
    The following two lemmas from \cite{Chung2006-xr} and \cite{adv_sampling} respectively explore useful properties of martingales to conclude the proof.
    \begin{lemma}\label{martingale old lemma}
        Let $X = (X_0, X_1, X_2 \dots X_n)$ be a martingale. Assume the variance of $X_i$ conditioned on previous elements $X_0, X_1 \dots X_{i-1}$ is bounded by $\sigma_i^2$ for $\sigma_0^2, \sigma_1^2 \dots \sigma_n^2 \geq 0$, and there exists some $M$ for which $|X_i - X_{i-1}| \leq M$ holds for all $i \in [0,n]$. Then for any $\lambda \geq 0$, the following inequality holds.
        \[Pr\left[ |X_n - X_0| \geq \lambda \right] \leq \exp \left( - \frac{\lambda^2}{2 \sum_{i=1}^n \sigma_i^2 + (M\lambda/3)}\right)\]
    \end{lemma}
    \begin{lemma}\label{martingale app lemma}
        The sequence $Z_0^R, Z_1^R \dots Z_n^R$ forms a martingale. The variance of $Z_i^R$ conditioned on the previous elements $Z_0^R \dots Z_{i-1}^R$ is bounded by $1/n^2p$, and $|Z_{i}^R - Z_{i-1}^R| \leq 1/np$.
    \end{lemma}
    Combining the two lemmas above, and assuming $np \geq \frac{9}{\epsilon^2}\ln{\delta/4}$, we have 
    \[Pr \left[ \left|A_n^R - B_n^R \right| \geq \frac{\epsilon}{2}\right] \leq 2 \exp \left( - \frac{(\epsilon/2)^2}{2n\cdot \frac{1}{n^2p} + \frac{\epsilon}{6np}}\right) < 2 \exp\left( - \frac{\epsilon^2 np}{9}\right) \leq \frac{\delta}{2}\]
    which yields the first of the two inequalities we try to prove. For the second inequality, observe that $B_n^R = \frac{|R \cap S|}{np} = \frac{|R \cap S|}{|S|} \cdot \frac{|S|}{np}$. Here we use the property of an \textit{oblivious sampler}: regardless of \sc{adversary}'s attack scheme, the oblivious samplers accept or reject an element based only on its index. In the case of a Bernoulli sampler, the size of $S$ is therefore described by a binomial distribution. By a Chernoff bound with $\delta = \epsilon/2$, we have 
    \[Pr\left[ ||S| - np| > \epsilon n p / 2\right] \leq 2 \exp\left( - \frac{\epsilon^2 np}{10}\right)\]
    The above inequality claims that with high probability the event $||S| - np| > \epsilon n p / 2$ \textit{does not} happen. Conditioning on $||S| - np| \geq \epsilon n p / 2$, the second inequality is completed with 
    \[\left| \frac{|R \cap S|}{|S|} - B_n^R \right|  = \left|\frac{|R \cap S|}{|S|} - \frac{|R \cap S|}{|S|} \cdot \frac{|S|}{np}\right| \leq \left| 1 - \frac{|S|}{np}\right| \leq \frac{\epsilon}{2} = \delta\] 
    
\end{proof}

\section{From Adversarial Sampling to Online Learnability}
\label{alon paper sec}
Throughout the above results, we have seen an interesting interplay between streaming and statistical learning theory. In particular, we note that  the learnability of certain hypothesis classes plays a nontrivial role in sampling theory. Recall from Definition \ref{ulln} and Theorem \ref{pac} that a hypothesis class $\mathcal{R}$ is PAC learnable if and only if the empirical error over large enough samples looks like the generalization error. This is equivalent to the statement that a large enough sample is an $\epsilon$-approximation of the distribution $\mathcal{D}$ with respect to $\mathcal{R}$ with high probability (here, we imagine a stream that perfectly encapsulates $\mathcal{D}$ that we wish to approximate with a sample). In the offline VC-theory there is an established equivalence between admittance of a Uniform Law of Large Numbers, $\epsilon$-approximability, and PAC learnability. It is therefore reasonable to inquire about a similar result in the setting of online learning, in which some aspect of sampling approximability relates to online learnability. 

On the flip side, we have also seen situations in which adversarial streaming research reveals shadows of online learning theory. For example, consider the adversarial attack on sampling that was devised in section \ref{adversarial attack on sampling}. At a high level, that attack constructed a large-depth mistake tree against the hypothesis class $\mathcal{R} = \{[1, b] : b \in \mathcal{U}\}$. The result was that a Bernoulli sampler would need to create a much larger sample to yield an $\epsilon$-approximation under this attack than it otherwise would.
In the statistical learning theory language, however, this is akin to demonstrating that $\mathcal{R}$, with a VC-dimension of 1, has a much higher Littlestone dimension to go along with its larger sample complexity when up against an adversary. This hints at a relationship between Littlestone dimension and sample complexity in adversarial settings. Even the functional form of Corollary \ref{corollary}, a result about sample size necessary for $\epsilon$-approximability in adversarial settings, feels very familiar in the context of statistical learning (c.f. Theorem \ref{quantitative pac}). 

Throughout our above investigations, all signs point to a fascinating and deeply-knit relationship between the study of adversarially-robust sampling (see Section \ref{robust sampling sec}) and measures of online learnability (Definitions \ref{online def realizable}, \ref{online def agnostic}). This relationship is \textit{precisely} what is formalized, proven, and explored in the recent work "Adversarial Laws of Large Numbers and
Optimal Regret in Online Classification" by Alon et al. \cite{Alon_ULLN}. Throughout the rest of this section and paper, we present the results of this paper and describe how they beautifully tie up some gaps in the previous theory and also set the stage for a fresh new area of research.  

\subsection{Adversarial Uniform Law of Large Numbers}
For the results below, we still work in the adversarial sampling model introducted in section \ref{adversarial setting sec}. Consider the following definition, meant to extend the concept of $\epsilon$-approximate sampling to the regime of adversarially-prepared streams.
\begin{definition}[Adversarial Uniform Law of Large Numbers]
Let $(\mathcal{U, R})$ be a set system. A hypothesis class $\mathcal{R}$ admits an Adversarial Uniform Law of Large Numbers if there exists a function $m_{\mathcal{R}}: (0, 1)^2 \rightarrow \mathbb{N}$ and a sampler such that for any adversarially-prepared stream X over $\mathcal{U}$ and any tolerances $\epsilon, \delta \in (0, 1)$, the sampler chooses at most $m_{\mathcal{R}}(\epsilon, \delta)$ samples from $X$ which form an $\epsilon$-approximation of $X$
    with probability at least $1 - \delta$. 
\end{definition}
Compare the above definition with our previous understanding about Uniform Laws of Large Numbers and forming $\epsilon$-approximations over regularly drawn (not adversarially-prepared) streams. This feels like a very natural mechanism with which to fill in the gaps and completely characterize online learning in the same way that we characterize offline learning with Theorems \ref{pac}, \ref{quantitative pac}. The two main theorems provided by Alon et al. do precisely that; we restate them below.
\begin{theorem}[Fundamental Theorem of Online Learning]
\label{online}
    Let $(\mathcal{U},\mathcal{R})$ be a set system. The following are equivalent:
    \begin{enumerate}
        \item $\mathcal{R}$ is online learnable in the realizable and agnostic settings.
        \item $\mathcal{R}$ admits an Adversarial Uniform Law of Large Numbers.
        \item $\mathcal{R}$ has a finite Littlestone dimension.
    \end{enumerate}
\end{theorem}

\begin{theorem}[Fundamental Theorem of Online Learning - Quantitiative]
\label{quantitative online}
    Let $(\mathcal{U},\mathcal{R})$ be a set system with finite Littlestone dimension $d$. Then, there exist constants $C_1, C_2$ such that:
    \begin{enumerate}
        \item $\mathcal{R}$'s Adversarial Uniform Law of Large Numbers has sample complexity 
        $$C_1 \frac{d}{\epsilon^2} \leq m_{\mathcal{R}}(\epsilon, \delta) \leq C_2 \frac{d + \log(1/\delta)}{\epsilon^2}$$
        \item $\mathcal{R}$ is online learnable in the agnostic setting with optimal regret at a time $T$ of
        $$C_1 \sqrt{d\cdot T} \leq R_T(\mathcal{R}) \leq C_2 \sqrt{d\cdot T}$$
    \end{enumerate}
\end{theorem}
Observe that Theorem \ref{online} completely characterizes online learnability in precisely the same way that Theorem \ref{pac} did for offline learnability. In particular, the Adversarial Uniform Law of Large Numbers plays the same role in online learnability as the Uniform Law of Large Numbers did offline, matching our intuitions at the beginning of this section. 

There are some interesting things to note about the quantitative results in Theorem \ref{quantitative online} as well. In the first bullet point there is an upper bound on adversarial sampling sample size that precisely matches the upper bound on regular sampling, simply with the VC-dimension replaced by the Littlestone dimension; this once again deepens the symmetry that has been developed throughout this survey. The lack of a tight lower bound, however, is interesting. While there is a tighter lower bound for $\epsilon$-nets instead of $\epsilon$-approximations, which correspond to realizable online learning instead of agnostic online learning (interested readers should read Section 7 of Alon et al.), neither bound matches the tightness we see in Theorem \ref{quantitative pac}. Perhaps there is further research to be done in order to improve the lower bound on adversarial sampling complexity, or perhaps there is much more variability in how difficult it is to achieve an Adversarial Uniform Law of Large Numbers for non-oblivious samplers (for all oblivious samplers like Bernoulli and Uniform, regular VC theory makes the bound tight). The second bullet point in Theorem \ref{quantitative online} tightens a until-now unsolved inequality for the optimal regret of agnostic online learners, definitively saying that optimal learners face regret of $\Theta(\sqrt{d \cdot T})$. In any case, these above results fill in crucial gaps in online learnability theory and strengthen the connections between sampling theory and statistical learning. In particular, the symmetry between oblivious sampling/offline learning and adversarial sampling/online learning is now made clear in a way that can hopefully inspire and enlighten future research.

\subsection{Overview of Proof Techniques}
In this final section, we highlight some of the key ideas and techniques used by Alon et al.\cite{Alon_ULLN} to prove the results of the previous section (Theorems \ref{online}, \ref{quantitative online}). The actual proof mechanics are wild and would require immense exposition, but there is still much to be learned about the connections between streaming and statistical learning from studying their high level approach. 

The first step of the proof of Theorem \ref{quantitative online} uses an online variant of a technique called \textit{double sampling}, a classic argument method credited to \cite{vc71}. In this application of double sampling (Section 8 in \cite{Alon_ULLN}), the authors replace the approximation error with respect to the entire stream with an approximation error with respect to a test set of the same size as the gathered sample. This has the effective result of reducing the size of the problem from gathering a sample of size $k$ from a stream of size $n$ to gathering a sample of size $k$ from a stream of size $2k$. Importantly, the adversary sees the $2k$ elements within the stream that are selected, but still has no more of an idea about what the final $k$ sampled elements will be. By doing so, the minimization of the $\epsilon$-approximation error reduces to minimizing the discrepancy in a well known online combinatorial game, defined below.
\begin{definition}[Online Combinatorial Discrepancy]
    The online discrepancy game with respect to $\mathcal{R}$ is a sequential game played between a painter and an adversary which proceeds as follows: at each round $t = 1, . . . , 2k$ the adversary places an item $x_t$ on the board, and the painter colors $x_t$ in either red or blue. The goal of the painter is that each set in $\mathcal{R}$ will be colored in a balanced fashion; i.e., if we denote by $I$ the set of indices of items from the stream $X$ that are colored red, the painter's goal is to minimize the discrepancy
    $$Disc_{2k}(\mathcal{R}, X, I) = \max_{R \in \mathcal{R}} \left| |X_I \cap R| - |X_{[2k] \setminus I} \cap R| \right| $$
\end{definition}
It is clear from the above definition that this is an equivalent formulation to the double sampling-reduced $\epsilon$-approximation optimization, since painting the item red or blue corresponds to sampling or not sampling the element, and coloring each set in $\mathcal{R}$ in a balanced fashion corresponds to preserving density in our sample. This reduction allows us to connect performance in combinatorial discrepancy (which in turn is connected to $\epsilon$-approximation) to a measure of complexity known as Sequential Rademacher Complexity \cite{sequential_radamacher}.
\begin{definition}
    The \textit{Sequential Rademacher Complexity} of a hypothesis class $\mathcal{R}$ is given by the expected discrepancy
    $$Rad_{n}(\mathcal{R})=\mathbb{E}\left[Disc_{n}(\mathcal{R}, X, I)\right]$$
\end{definition}
This is a useful direction to maneuver toward because of an already established result that the optimal regret of agnostic online learning satisfies $R_T(\mathcal{R}) \leq 2 Rad_T(\mathcal{R})$ (for more details, see Section 12 of \cite{Alon_ULLN}). In addition, it ties the sample complexity of $\epsilon$-approximation in the adversarial setting directly to Sequential Rademacher Complexity, leading us to the first item of the quantitative theorem. So, in order to arrive at Theorem \ref{quantitative online}, all that is left is to do is bound $Rad_{T}(\mathcal{R})$ by $O(\sqrt{Ldim(\mathcal{R}) \cdot T})$. The actual mechanics behind this involve fractional $(\epsilon, \gamma)$-covers of hypothesis classes $\mathcal{R}$ (probability measures over dynamic sets that agrees with most of $\mathcal{R}$) and a chaining argument, but the tooling for that part of the proof is beyond the scope of this survey (the details can be found in Section 9 of \cite{Alon_ULLN} for those interested). All in all, this high level overview of the proof is presented to once again emphasize the depth and richness of the connection between adversarial sampling and online learning theory, this time through the lens of the Sequential Rademacher Complexity.

\section{Remarks and Future Research}
We believe that the relationship between two until-recently disjoint fields of study --- online statistical learning and adversarially-robust streaming ---  that we have explored in this paper creates a new area of research with much promise and direction. Some particular directions of further research could pertain to the study of non-oblivious samplers that are created using the robustification methods of \cite{linear_sketch}, \cite{windows}, and others. These methods for mechanically turning an arbitrary oblivious streaming algorithm into an adversarially-robust one could result in new sampling schema that yield better performance on adversarial $\epsilon$-approximation, different bounds on Sequential Rademacher Complexity, and in general could produce different results/perspectives than those found in Alon et al. using oblivious samplers. In the other direction, perhaps the bounds and connections to Rademacher Complexity theory could enlighten designers of sampling algorithms to devise techniques that are adversarially-robust and have practical standalone value in their application as sampling methods. There seem to be many promising directions to go to both build on and make use of the new connections explored throughout this research.

\section{Acknowledgements}
We would like to thank Professors Matthew Weinberg and Huacheng Yu for their kind feedback, reliable support, and incredible teaching. Thank you!

\newpage
\printbibliography

\end{document}